\def\year{2019}\relax
\documentclass[letterpaper]{article}
\usepackage{aaai19}  
\usepackage[utf8]{inputenc} 
\usepackage{times}  
\usepackage{helvet}  
\usepackage{courier} 
\usepackage{graphicx}  
\usepackage{todonotes}
\usepackage{amsmath}
\usepackage{amsthm}
\usepackage{amssymb}
\usepackage{url}
\usepackage[capitalise,noabbrev]{cleveref}
\usepackage{mathtools}
\usepackage{mleftright}
\usepackage{thm-restate}
\usepackage{subcaption}
\usepackage{scrextend}
\usepackage{scrextend}

\usepackage{booktabs}      
\usepackage{amsfonts}      
\usepackage{nicefrac}      
\usepackage[tracking=true, stretch=0,shrink=55]{microtype} 
\usepackage{xcolor}
\usepackage{multirow}

\usepackage{enumitem}
\setlist[enumerate]{noitemsep, topsep=0.5\topsep}
\setlist[description]{noitemsep, topsep=0.5\topsep}
\setlist[itemize]{noitemsep, topsep=0.5\topsep}

\theoremstyle{definition}
\newtheorem{theorem}{Theorem}
\newtheorem{proposition}[theorem]{Proposition}
\newtheorem{lemma}[theorem]{Lemma}
\newtheorem{corollary}[theorem]{Corollary}

\newtheorem{example}[theorem]{Example}

\newtheorem{remark}[theorem]{Remark}

\definecolor{green}{rgb}{0.12, 0.66, 0.0}
\definecolor{blue}{rgb}{0.0, 0.25, 0.65} 
\newcommand{\xhdr}[1]{{\noindent\bfseries #1}.}

\newcommand{\new}[1]{\emph{#1}}

\newcommand{\bbR}{\ensuremath{\mathbb{R}}}

\newcommand{\RR}{\mathbb{R}}

\newcommand{\ZZ}{\mathbb{Z}}

\newcommand{\NN}{\mathbb{N}}

\renewcommand{\vec}[1]{\mathbf{#1}} 

\newcommand{\oms}{\{\!\!\{}
\newcommand{\cms}{\}\!\!\}}

\newcommand{\sgn}{\operatorname{sgn}}

\begin{document}

\title{Weisfeiler and Leman Go Neural: Higher-order Graph Neural Networks}
\author{
	Christopher Morris\textsuperscript{1},
	Martin Ritzert\textsuperscript{2},
	Matthias Fey\textsuperscript{1},
	William L. Hamilton\textsuperscript{3},\\\bf \Large
	Jan Eric Lenssen\textsuperscript{1},
	Gaurav Rattan\textsuperscript{2},
	Martin Grohe\textsuperscript{2}\\
	\textsuperscript{1}{TU Dortmund University}\\
	\textsuperscript{2}{RWTH Aachen University}\\
	\textsuperscript{3}{McGill University and MILA}\\\
	\{christopher.morris, matthias.fey, janeric.lenssen\}@tu-dortmund.de,\\
	\{ritzert, rattan, grohe\}@informatik.rwth-aachen.de,\\
	wlh@cs.mcgill.ca
}

\maketitle 

\begin{abstract}
	In recent years, graph neural networks (GNNs) have emerged as a powerful neural architecture to learn vector representations of nodes and graphs in a supervised, end-to-end fashion. Up to now, GNNs have only been evaluated empirically---showing promising results. The following work investigates GNNs from a theoretical point of view and relates them to the $1$-dimensional Weisfeiler-Leman graph isomorphism heuristic ($1$-WL). We show that GNNs have the same expressiveness as the $1$-WL in terms of distinguishing non-isomorphic (sub-)graphs. Hence, both algorithms also have the same shortcomings. Based on this, we propose a generalization of GNNs, so-called $k$-dimensional GNNs ($k$-GNNs), which can take higher-order graph structures at multiple scales into account. These higher-order structures play an essential role in the characterization of social networks and molecule graphs. Our experimental evaluation confirms our theoretical findings as well as confirms that higher-order information is useful in the task of graph classification and regression.
\end{abstract}

\section{Introduction}
Graph-structured data is ubiquitous across application domains ranging from chemo- and bioinformatics to image and social network analysis.
To develop successful machine learning models in these domains, we need techniques that can exploit the rich information inherent in graph structure, as well as the feature information contained within a graph's nodes and edges. In recent years, numerous approaches have been proposed for machine learning graphs---most notably, approaches based on graph kernels \cite{Vis+2010} or, alternatively, using graph neural network algorithms \cite{Ham+2017a}.

Kernel approaches typically fix a set of features in advance---e.g., indicator features over subgraph structures or features of local node neighborhoods.
For example, one of the most successful kernel approaches, the \new{Weisfeiler-Lehman subtree kernel}~\cite{She+2011}, which is based on the $1$-dimensional Weisfeiler-Leman graph isomorphism heuristic \cite[pp.\,79\,ff.]{Gro2017}, generates node features through an iterative relabeling, or \emph{coloring}, scheme: 
First, all nodes are assigned a common initial color; the algorithm then iteratively recolors a node by aggregating over the multiset of colors in its neighborhood, and the final feature representation of a graph is the histogram of the resulting node colors.  
By iteratively aggregating over local node neighborhoods in this way, the WL subtree kernel is able to effectively summarize the neighborhood substructures present in a graph. 
However, while powerful, the WL subtree kernel---like other kernel methods---is limited because this feature construction scheme is fixed (i.e., it does not adapt to the given data distribution). Moreover, this approach---like the majority of kernel methods---focuses only on the graph structure and cannot interpret continuous node and edge labels, such as real-valued vectors which play an important role in applications such as bio- and chemoinformatics. 
 
Graph neural networks (GNNs) have emerged as a machine learning framework addressing the above challenges.
Standard GNNs can be viewed as a neural version of the $1$-WL algorithm, where colors are replaced by continuous feature vectors and neural networks are used to aggregate over node neighborhoods \cite{Ham+2017,Kip+2017}. 
In effect, the GNN framework can be viewed as implementing a continuous form of graph-based ``message passing'', where local neighborhood information is aggregated and passed on to the neighbors~\cite{Gil+2017}. By deploying a trainable neural network to aggregate information in local node neighborhoods, GNNs can be trained in an end-to-end fashion together with the parameters of the classification or regression algorithm, possibly allowing for greater adaptability and better generalization 
compared to the kernel counterpart of the classical $1$-WL algorithm. 

Up to now, the evaluation and analysis of GNNs has been largely empirical, showing promising results compared to kernel approaches, see, e.g.,~\cite{Yin+2018}. However, it remains unclear how GNNs are actually encoding graph structure information into their vector representations, and whether there are theoretical advantages of GNNs compared to kernel based approaches. 

\xhdr{Present Work}
We offer a theoretical exploration of the relationship between GNNs and kernels that are based on the $1$-WL algorithm. 
We show that GNNs cannot be more powerful than the $1$-WL in terms of distinguishing non-isomorphic (sub-)graphs, e.g., the properties of subgraphs around each node. 
This result holds for a broad class of GNN architectures and all possible choices of parameters for them.
On the positive side, we show that given the right parameter initialization GNNs have the same expressiveness as the $1$-WL algorithm, completing the equivalence. 
Since the power of the $1$-WL has been completely characterized, see, e.g.,~\cite{Arv+2015,kiefer2015graphs}, we can transfer these results to the case of GNNs, showing that both approaches have the same shortcomings.

Going further, we leverage these theoretical relationships to propose a generalization of GNNs, called $k$-GNNs, which are neural architectures based on the $k$-dimensional WL algorithm ($k$-WL), which are strictly more powerful than GNNs. 
The key insight in these higher-dimensional variants is that they perform message passing directly between subgraph structures, rather than individual nodes.
This higher-order form of message passing can capture structural information that is not visible at the node-level.

Graph kernels based on the $k$-WL have been proposed in the past \cite{Mor+2017}.
However, a key advantage of implementing higher-order message passing in GNNs---which we demonstrate here---is that we can design hierarchical variants of $k$-GNNs, which combine graph representations learned at different granularities in an end-to-end trainable framework. 
Concretely, in the presented hierarchical approach the initial messages in a $k$-GNN are based on the output of lower-dimensional $k'$-GNN (with $k' < k$), which allows the model to effectively capture graph structures of varying granularity.  Many real-world graphs inherit a hierarchical structure---e.g., in a social network we must model both the ego-networks around individual nodes, as well as the coarse-grained relationships between entire communities, see, e.g.,~\cite{New2003}---and our experimental results demonstrate that these hierarchical $k$-GNNs are able to consistently outperform traditional GNNs on a variety of graph classification and regression tasks. Across twelve graph regression tasks from the QM9 benchmark, we find that our hierarchical model reduces the mean absolute error by 54.45\% on average. For graph classification, we find that our hierarchical models leads to slight performance gains.

\xhdr{Key Contributions}
Our key contributions are summarized as follows:
\begin{enumerate}
	\item We show that GNNs are not more powerful than the $1$-WL in terms of distinguishing non-isomorphic (sub-)graphs. Moreover, we show that, assuming a suitable parameter initialization, GNNs have the same power as the $1$-WL.
	\item We propose $k$-GNNs, which are strictly more powerful than GNNs. Moreover, we propose a hierarchical version of $k$-GNNs, so-called $1$-$k$-GNNs, which are able to work with the fine- and coarse-grained structures of a given graph, and relationships between those. 
	\item Our theoretical findings are backed-up by an experimental study, showing that higher-order graph properties are important for successful graph classification and regression.
\end{enumerate}

\section{Related Work}
Our study builds upon a wealth of work at the intersection of supervised learning on graphs, kernel methods, and graph neural networks. 

Historically, kernel methods---which implicitly or explicitly map graphs to elements of a Hilbert space---have been the dominant approach for supervised learning on graphs. 
Important early work in this area includes random-walk based kernels \cite{Gae+2003,Kas+2003}) and kernels based on shortest paths \cite{Borgwardt2005}.
More recently, developments in graph kernels have emphasized scalability, focusing on techniques that bypass expensive Gram matrix computations by using explicit feature maps.
Prominent examples of this trend include kernels based on graphlet counting~\cite{She+2009}, and, most notably, the Weisfeiler-Lehman subtree kernel~\cite{She+2011} as well as its higher-order variants~\cite{Mor+2017}. 
Graphlet and Weisfeiler-Leman kernels have been successfully employed within frameworks for smoothed and deep graph kernels~\cite{Yan+2015,Yan+2015a}. Recent works focus on assignment-based approaches~\cite{Kri+2016,Nik+2017,Joh+2015}, spectral approaches~\cite{Kon+2016}, and graph decomposition approaches~\cite{Nik+2018}.
Graph kernels were dominant in graph classification for several years, leading to new state-of-the-art results on many classification tasks. 
However, they are limited by the fact that they cannot effectively adapt their feature representations to a given data distribution, since they generally rely on a fixed set of features. More recently, a number of approaches to graph classification based upon neural networks have been proposed. 
Most of the neural approaches fit into the graph neural network framework proposed by~\cite{Gil+2017}. Notable instances of this model include \new{Neural Fingerprints}~\cite{Duv+2015}, \emph{Gated Graph Neural Networks}~\cite{Li+2016}, \emph{GraphSAGE}~\cite{Ham+2017}, \emph{SplineCNN}~\cite{Fey+2018}, and the spectral approaches proposed in~\cite{Bru+2014,Def+2015,Kip+2017}---all of which descend from early work in~\cite{Mer+2005} and~\cite{Sca+2009}.
Recent extensions and improvements to the GNN framework include approaches to incorporate different local structures around subgraphs \cite{Xu+2018} and novel techniques for pooling node representations in order perform graph classification \cite{Zha+2018,Yin+2018}.
GNNs have achieved state-of-the-art performance on several graph classification benchmarks in recent years, see, e.g.,~\cite{Yin+2018}---as well as applications such as protein-protein interaction prediction~\cite{Fou+2017}, recommender systems~\cite{Yin+2018a}, and the analysis of quantum interactions in molecules~\cite{Sch+2017}.
A survey of recent advancements in GNN techniques can be found in \cite{Ham+2017a}.

Up to this point (and despite their empirical success) there has been very little theoretical work on GNNs---with the notable exceptions of Li et\ al.'s \cite{Li+2018a} work connecting GNNs to a special form Laplacian smoothing and Lei et al.'s\@ \cite{Lei+2017} work showing that the feature maps generated by GNNs lie in the same Hilbert space as some popular graph kernels. Moreover, Scarselli et al.\ \cite{Sca+2009a} investigates the approximation capabilities of GNNs. 

\section{Preliminaries}\label{prelim}

We start by fixing notation, and then outline the Weisfeiler-Leman algorithm and the standard graph neural network framework.  

\subsection{Notation and Background}
A \new{graph} $G$ is a pair $(V,E)$ with a finite set of \new{nodes} $V$ and a set of \new{edges} $E \subseteq \{ \{u,v\} \subseteq V \mid u \neq v \}$. We denote the set of nodes and the set of edges of $G$ by $V(G)$ and $E(G)$, respectively. For ease of notation we denote the edge $\{u,v\}$ in $E(G)$ by $(u,v)$ or $(v,u)$.
Moreover, $N(v)$ denotes the \new{neighborhood} of $v$ in $V(G)$, i.e., $N(v) = \{ u \in V(G) \mid (v, u) \in E(G) \}$. We say that two graphs $G$ and $H$ are \new{isomorphic} if there exists an edge preserving bijection $\varphi: V(G) \to V(H)$, i.e., $(u,v)$ is in $E(G)$ if and only if $(\varphi(u),\varphi(v))$ is in $E(H)$. We write $G \simeq H$ and call the equivalence classes induced by $\simeq$ \new{isomorphism types}. Let $S \subseteq V(G)$ then $G[S] = (S,E_S)$ is the \new{subgraph induced} by $S$ with $E_S = \{ (u,v) \in E(G) \mid u,v \in S \}$. A \new{node coloring} is a function $V(G) \to \Sigma$ with arbitrary codomain $\Sigma$. Then a  \new{node colored} or \new{labeled graph}  $(G,l)$ is a graph $G$ endowed with a node coloring $l \colon V(G) \to \Sigma$. We say that $l(v)$ is a \new{label} or \new{color} of $v\in V(G)$. We say that a node coloring $c$ \new{refines} a node coloring $d$, written $c \sqsubseteq d$, if $c(v) = c(w)$ implies $d(v) = d(w)$ for every $v,w$ in $V(G)$. 
Two colorings are \new{equivalent} if $c \sqsubseteq d$ and $d \sqsubseteq c$, and we write $c \equiv d$.
A \new{color class} $Q\subseteq V(G)$ of a node coloring $c$ is a maximal set of nodes with $c(v)=c(w)$ for every $v,w$ in $Q$. Moreover, let $[1\!:\!n] = \{ 1, \dotsc, n \} \subset \NN$ for $n > 1$, let $S$ be a set then the set of \new{$k$-sets} $[S]^k = \{ U \subseteq S \mid |U| = k \}$  for $k \geq 2$, which is the set of all subsets with cardinality $k$, and let $\oms \dots \cms$ denote a multiset.

\subsection{Weisfeiler-Leman Algorithm}
We now describe the \textsc{$1$-WL} algorithm for labeled graphs. Let $(G,l)$ be a labeled graph. In each iteration, $t \geq 0$, the $1$-WL computes a node coloring $c^{(t)}_l \colon V(G) \to \Sigma$,
which depends on the coloring from the previous iteration.
In iteration $0$, we set $c^{(0)}_l = l$. Now in iteration $t>0$, we set 
\begin{equation}\label{eq:wlColoring}
	c_{l}^{(t)}(v) \!=\! \textsc{hash}\Big(\!\big(c_l^{(t-1)}(v),\oms c_l^{(t-1)}(u)\!\mid\!u \in\!N(v) \!\cms \big)\! \Big)
\end{equation}
where $\textsc{hash}$ bijectively maps the above pair to a unique value in $\Sigma$, which has not been used in previous iterations. To test two graph $G$ and $H$ for isomorphism, we run the above algorithm in ``parallel'' on both graphs. Now if the two graphs have a different number of nodes colored $\sigma$ in $\Sigma$, the \textsc{$1$-WL} concludes that the graphs are not isomorphic. Moreover, if the number of colors between two iterations does not change, i.e., the cardinalities of the images of $c_l^{(t-1)}$ and $c_l^{(t)}$ are equal, the algorithm terminates. Termination is guaranteed after at most $\max \{ |V(G)|,|V(H)| \}$ iterations. It is easy to see that the algorithm is not able to distinguish all non-isomorphic graphs, e.g., see~\cite{Cai+1992}. Nonetheless, it is a powerful heuristic, which can successfully test isomorphism for a broad class of graphs~\cite{Bab+1979}.

The $k$-dimensional Weisfeiler-Leman algorithm ($k$-WL), for $k \geq 2$, is a generalization of the $1$-WL which colors tuples from $V(G)^k$ instead of nodes. That is, the algorithm computes a coloring $c^{(t)}_{l,k} \colon V(G)^k \to \Sigma$. In order to describe the algorithm, we define the $j$-th neighborhood
\begin{equation}\label{gnei}
	N_j(s) \!=\! \{ ( s_1, \dotsc, s_{j-1}, r, s_{j+1}, \dotsc, s_k) \mid r \in V(G) \}
\end{equation}
of a $k$-tuple $s = (s_1, \dotsc, s_k )$ in $V(G)^k$. That is, the $j$-th neighborhood $N_j(t)$ of $s$ is obtained by replacing the $j$-th component of $s$ by every node from $V(G)$. In iteration $0$, the algorithm labels each $k$-tuple with its \new{atomic type}, i.e., two $k$-tuples $s$ and $s'$ in $V(G)^k$ get the same color if the map $s_i \mapsto s'_i$ induces a (labeled) isomorphism between the subgraphs induced from the nodes from $s$ and $s'$, respectively. For iteration $t > 0$, we define 
\begin{equation}\label{wl-prim}
	C^{(t)}_j(s) = \textsc{hash}_{}\big(\oms c^{(t-1)}_{l,k}(s') \mid s' \in N_j(s)\cms\big), 
\end{equation}
and set 
\begin{equation}\label{labelk}
	c_{k,l}^{(t)}(s)\!=\! \textsc{hash}_{}\Big( \!\big(c_{k,l}^{(t-1)}(s), \big( C^{(t)}_1(s), \dots, C^{(t)}_k(s)  \big) \! \Big)\,.
\end{equation}

Hence, two tuples $s$ and $s'$ with $c_{k,l}^{(t-1)}(s) = c_{k,l}^{(t-1)}(s')$ get different colors in iteration $t$ if there exists $j$ in $[1\!:\!k]$ such that the number of $j$-neighbors of $s$ and $s'$, respectively, colored with a certain color is different. 
The algorithm then proceeds analogously to the \textsc{$1$-WL}. 
By increasing $k$, the algorithm gets more powerful in terms of distinguishing non-isomorphic graphs, i.e., for each $k\geq 2$, there are non-isomorphic graphs which can be distinguished by the ($k+1$)-WL but not by the $k$-WL~\cite{Cai+1992}.
We note here that the above variant is not equal to the \emph{folklore} variant of $k$-WL described in~\cite{Cai+1992}, which differs slightly in its update rule. 
However, it holds that the $k$-WL using~\cref{labelk} is as powerful as the folklore $(k\!-\!1)$-WL \cite{GroheO15}.

\xhdr{WL Kernels}
After running the WL algorithm, the concatenation of the histogram of colors in each iteration can be used as a feature vector in a kernel computation. 
Specifically, in the histogram for every color $\sigma$ in $\Sigma$ there is an entry containing the number of nodes or $k$-tuples that are colored with $\sigma$.

\subsection{Graph Neural Networks}
Let $(G,l)$ be a labeled graph with an initial node coloring $f^{(0)}: V(G)\rightarrow \RR^{1\times d}$ that is \emph{consistent} with $l$.
This means that each node $v$ is annotated with a feature $f^{(0)}(v)$ in $\bbR^{1\times d}$ such that $f^{(0)}(u) = f^{(0)}(v)$ if and only if $l(u) = l(v)$.
Alternatively, $f^{(0)}(v)$ can be an arbitrary  real-valued feature vector associated with $v$.
Examples include continuous atomic properties in chemoinformatic applications where nodes correspond to atoms, or vector representations of text in social network applications. 
A GNN model consists of a stack of neural network layers, where each layer aggregates local neighborhood information, i.e., features of neighbors, around each node and then passes this aggregated information on to the next layer. 

A basic GNN model can be implemented as follows~\cite{Ham+2017a}.
In each layer $t > 0$,  we compute a new feature 
\begin{equation}\label{eq:basicgnn}
	f^{(t)}(v) = \sigma \Big( f^{(t-1)}(v) \cdot  W^{(t)}_1 +\, \sum_{\mathclap{w \in N(v)}}\,\, f^{(t-1)}(w) \cdot W_2^{(t)} \Big)
\end{equation}
in  $\bbR^{1 \times e}$ for $v$, where 
$W_1^{(t)}$ and $W_2^{(t)}$ are parameter matrices from $\bbR^{d \times e}$
, and $\sigma$ denotes a component-wise non-linear function, e.g., a sigmoid or a ReLU.\footnote{For clarity of presentation we omit biases.}

Following~\cite{Gil+2017}, one may also replace the sum defined over the neighborhood in the above equation by a permutation-invariant, differentiable function, and one may substitute the outer sum, e.g., by a column-wise vector concatenation or LSTM-style update step.
Thus, in full generality a new feature $f^{(t)}(v)$ is computed as
\begin{equation}\label{eq:gnngeneral}
	f^{W_1}_{\text{merge}}\Big(f^{(t-1)}(v) ,f^{W_2}_{\text{aggr}}\big(\oms f^{(t-1)}(w) \mid  w \in N(v)\cms \big)\!\Big),
\end{equation}
where $f^{W_1}_{\text{aggr}}$ aggregates over the set of neighborhood features and $f^{W_2}_{\text{merge}}$ merges the node's representations from step $(t-1)$ with the computed neighborhood features.
Both $f^{W_1}_{\text{aggr}}$ and $f^{W_2}_{\text{merge}}$ may be arbitrary differentiable, permutation-invariant functions (e.g., neural networks), and, by analogy to Equation \ref{eq:basicgnn}, we denote their parameters as $W_1$ and $W_2$, respectively. 
In the rest of this paper, we refer to neural architectures implementing~\cref{eq:gnngeneral} as \emph{$1$-dimensional GNN architectures} ($1$-GNNs).

A vector representation $f_{\text{GNN}}$ over the whole graph can be computed by summing over the vector representations computed for all nodes, i.e.,  
\begin{equation*}
	f_{\text{GNN}}(G) = \sum_{v \in V(G)} f^{(T)}(v),
\end{equation*}
where $T > 0$ denotes the last layer. More refined approaches use differential pooling operators based on sorting~\cite{Zha+2018} and soft assignments~\cite{Yin+2018}. 

In order to adapt the parameters $W_1$ and $W_2$ of~\cref{eq:basicgnn,eq:gnngeneral}, to a given data distribution, they are optimized in an end-to-end  fashion (usually via stochastic gradient descent) together with the parameters of a neural network used for classification or regression.

\section{Relationship Between 1-WL and 1-GNNs}

In the following we explore the relationship between the $1$-WL and $1$-GNNs. 
Let $(G,l)$ be a labeled graph, and let $\mathbf{W}^{(t)} = \big(W^{(t')}_1, W^{(t')}_2 \big)_{t'\leq t}$ denote the GNN parameters given by \cref{eq:basicgnn} or~\cref{eq:gnngeneral} up to iteration $t$. 
We encode the initial labels $l(v)$ by vectors  $f^{(0)}(v)\in\RR^{1\times d}$, e.g., using a $1$-hot encoding. 

Our first theoretical result shows that the $1$-GNN architectures do not have more power in terms of distinguishing between non-isomorphic (sub-)graphs than the $1$-WL algorithm.
More formally, let $f^{W_1}_{\text{merge}}$ and $f^{W_2}_{\text{aggr}}$ be any two functions chosen in \eqref{eq:gnngeneral}.
For every encoding of the labels $l(v)$ as vectors $f^{(0)}(v)$, and for every choice of $\mathbf{W}^{(t)}$, we have that the coloring $c^{(t)}_l$ of $1$-WL always refines the coloring $f^{(t)}$ induced by a $1$-GNN parameterized by $\mathbf{W}^{(t)}$.

\begin{theorem}\label{thm:refine}
	Let $(G, l)$ be a labeled graph. Then for all $t\ge 0$ and for all choices of initial colorings $f^{(0)}$ consistent with $l$, and weights $\mathbf{W}^{(t)}$,
	\begin{equation*}
		c^{(t)}_l \sqsubseteq f^{(t)}\,.
	\end{equation*}
\end{theorem}

Our second result states that there exist a sequence of parameter matrices $\mathbf{W}^{(t)}$ such that $1$-GNNs have exactly the same power in terms of distinguishing non-isomorphic \mbox{(sub-)}graphs as the $1$-WL algorithm.
This even holds for the simple architecture~\eqref{eq:basicgnn}, provided we choose the encoding of the initial labeling $l$ in such a way that different labels are encoded by linearly independent vectors.

\begin{theorem}\label{equal}
	Let $(G, l)$ be a labeled graph. Then for all \mbox{$t\geq 0$} there exists a sequence of weights $\mathbf{W}^{(t)}$, and a $1$-GNN architecture such that 
	\begin{equation*}
		c^{(t)}_l \equiv f^{(t)}\,.
	\end{equation*}
\end{theorem}

Hence, in the light of the above results, $1$-GNNs may viewed as an extension of the $1$-WL which in principle have the same power but are more flexible in their ability to adapt to the learning task at hand and are able to handle continuous node features.

\subsection{Shortcomings of Both Approaches}

The power of $1$-WL has been completely characterized, see, e.g.,~\cite{Arv+2015}.  
Hence, by using~\cref{thm:refine,equal}, this characterization is also applicable to $1$-GNNs. 
On the other hand, $1$-GNNs have the same shortcomings as the $1$-WL. 
For example, both methods will give the same color to every node in a graph consisting of a triangle and a $4$-cycle, although vertices from the triangle and the vertices from the $4$-cycle are clearly different.
Moreover, they are not capable of capturing simple graph theoretic properties, e.g., triangle counts, which are an important measure in social network analysis~\cite{Mil+2002,New2003}.

\section{$\boldsymbol{k}$-dimensional Graph Neural Networks}

\begin{figure*}[t]
	\centering
	\begin{subfigure}[b]{0.65\linewidth}
		\centering
		\includegraphics[width=0.7\textwidth]{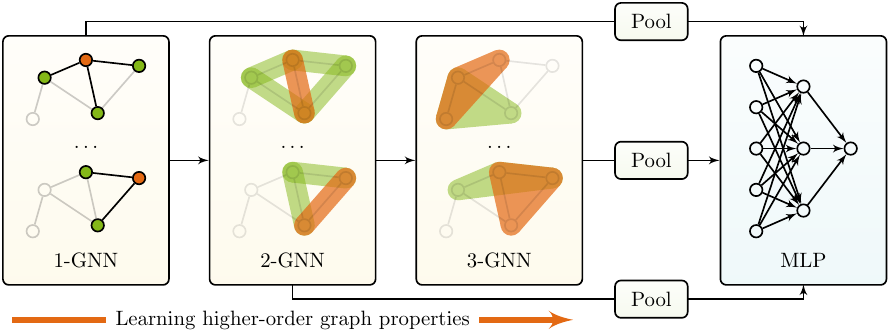}
		\caption{Hierarchical 1-2-3-GNN network architecture}\label{fig:architecture}
	\end{subfigure}
	\hspace{-.5cm}
	\begin{subfigure}[b]{0.28\linewidth}
		\centering
		\includegraphics[width=0.7\textwidth]{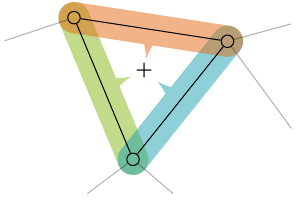}
		\caption{Pooling from $2$- to $3$-GNN.}\label{fig:pooling}
	\end{subfigure}
	\caption{Illustration of the proposed hierarchical variant of the $k$-GNN layer. For each subgraph $S$ on $k$ nodes a feature $f$ is learned, which is initialized with the learned features of all $(k-1)$-element subgraphs of $S$. Hence, a hierarchical representation of the input graph is learned.}\label{fig:overview}
\end{figure*}

In the following, we propose a generalization of $1$-GNNs, so-called $k$-GNNs, which are based on the $k$-WL. Due to scalability and limited GPU memory, we consider a set-based version of the $k$-WL. For a given $k$, we consider all $k$-element subsets $[V(G)]^k$ over $V(G)$. Let $s =  \{ s_1, \dotsc, s_k \}$ be a $k$-set in $[V(G)]^k$, then we define the \emph{neighborhood} of $s$ as 
\begin{equation*}\label{eq:localNeighborhood}
	N(s) = \{ t\in [V(G)]^k\mid |s\cap t|=k-1\}\,.
\end{equation*}
The \emph{local neighborhood}\, $N_L(s)$ consists of all $t\in N(s)$ such that $(v,w)\in E(G)$ for the unique $v\in s\setminus t$ and the unique $w\in t\setminus s$. The \emph{global neighborhood} $N_G(s)$ then is defined as $N(s) \setminus N_{\text{L}}(s)$.\footnote{Note that the definition of the local neighborhood is different from the the one defined in~\cite{Mor+2017} which is a superset of our definition. 
Our computations therefore involve sparser graphs.}

The set based $k$-WL works analogously to the $k$-WL, i.e., it computes a coloring $c_{\text{s},k,l}^{(t)} : [V(G)]^k \to \Sigma$ as in \cref{eq:wlColoring} based on the above neighborhood. 
Initially, $c_{\text{s},k,l}^{(0)}$ colors each element $s$ in $[V(G)]^k$ with the isomorphism type of $G[s]$. 

Let $(G,l)$ be a labeled graph. 
In each $k$-GNN layer $t \geq 0$, we compute a feature vector $f^{(t)}_{k}(s)$ for each $k$-set $s$ in $[V(G)]^k$. 
For $t=0$, we set $f^{(0)}_{k}(s)$  to  $f^{\text{iso}}(s)$, a one-hot encoding  of the isomorphism type of $G[s]$ labeled by $l$. In each layer $t > 0$,  we compute new features by 
\begin{align*}
	f^{(t)}_{k}(s) =  \sigma & \Big(  f^{(t-1)}_{k}(s) \cdot W_1^{(t)} +\, \sum_{\mathclap{u \in N_L (s) \cup N_G (s)}} \, f^{(t-1)}_{k}(u) \cdot W_2^{(t)}\Big)\,. 
\end{align*}
Moreover, one could split the sum into two sums ranging over $N_L (s)$ and $N_G (s)$ respectively, using distinct parameter matrices to enable the model to learn the importance of local and global neighborhoods. 
To scale $k$-GNNs to larger datasets and to prevent overfitting, we propose \emph{local} $k$-GNNs, where we omit the global neighborhood of  $s$, i.e.,
\begin{equation*}
	f^{(t)}_{k,\text{L}}(s) = \sigma \Big( f^{(t-1)}_{k,\text{L}}(s) \cdot W^{(t)}_1 + \,\sum_{\mathclap{u \in N_{L}(s)}} \, f^{(t-1)}_{k,\text{L}}(u) \cdot  W^{(t)}_2  \Big)\,.
\end{equation*}
The running time for evaluation of the above depends on $|V|$, $k$ and the sparsity of the graph (each iteration can be bounded by the number of subsets of size $k$ times the maximum degree). Note that we can scale our method to larger datasets by using sampling strategies introduced in, e.g.,~\cite{Mor+2017,Ham+2017}. We can now lift the results of the previous section to the $k$-dimensional case. 
\begin{proposition}\label{pro:refines}
	Let $(G, l)$ be a labeled graph and let $k\geq 2$. Then for all $t \ge 0$, for all choices of initial colorings $f_k^{(0)}$ consistent with $l$ and for all weights $\mathbf{W}^{(t)}$,
	\begin{equation*}
		c^{(t)}_{\text{s},k,l} \sqsubseteq f^{(t)}_{k}\,. 
	\end{equation*}
\end{proposition}
Again the second result states that there exists a suitable initialization of the parameter matrices $\mathbf{W}^{(t)}$ such that $k$-GNNs have exactly the same power in terms of distinguishing non-isomorphic (sub-)graphs as the set-based $k$-WL.
\begin{proposition}\label{pro:equality}
	Let $(G, l)$ be a labeled graph and let $k\geq 2$. Then for all $t \geq 0$ there exists a sequence of weights $\mathbf{W}^{(t)}$, and a $k$-GNN architecture such that   
	\begin{equation*}
		c^{(t)}_{\text{s},k,l} \equiv f^{(t)}_{k}\,.
	\end{equation*}		
\end{proposition}

\subsection{Hierarchical Variant}

One key benefit of the end-to-end trainable $k$-GNN frame\-work---compared to the discrete $k$-WL algorithm---is that we can hierarchically combine representations learned at different granularities.
Concretely, rather than simply using one-hot indicator vectors as initial feature inputs in a $k$-GNN, we propose a \emph{hierarchical} variant of $k$-GNN that uses the features learned by a $(k-1)$-dimensional GNN, in addition to the (labeled) isomorphism type, as the initial features, i.e.,
\begin{equation*}
	f^{(0)}_{k}(s) =  \sigma \Big(\Big[f^{\text{iso}}(s), \sum_{u \subset s} f^{\,(T_{k-1})}_{k-1}(u) \Big] \cdot W_{k-1} \Big),
\end{equation*}
for some $T_{k-1} > 0$, where $W_{k-1}$ is a matrix of appropriate size, and square brackets denote matrix concatenation. 

Hence, the features are recursively learned from dimensions $1$ to $k$ in an end-to-end fashion.
This hierarchical model also satisfies \cref{pro:refines,pro:equality}, so its representational capacity is theoretically equivalent to a standard $k$-GNN (in terms of its relationship to $k$-WL).
Nonetheless, hierarchy is a natural inductive bias for graph modeling, since many real-world graphs incorporate hierarchical structure, so we expect this hierarchical formulation to offer empirical utility. 
\begin{table*}[t]\
	\caption{Classification accuracies in percent on various graph benchmark datasets.}
	\label{fig:classification_results}
	\renewcommand{\arraystretch}{0.90}
	\centering
	\begin{tabular}{@{}clccccccc@{}}
		\toprule
		& \multirow{3}{*}{\vspace*{8pt}\textbf{Method}} & \multicolumn{7}{c}{\textbf{Dataset}} \\
		\cmidrule{3-9}
		&                                    & {\textsc{Pro}} & {\textsc{IMDB-Bin}} & \!{\textsc{IMDB-Mul}} & \!{\textsc{PTC-FM}} & \!{\textsc{NCI1}} & \!{\textsc{Mutag}} & \!{\textsc{PTC-MR}} \\
		\cmidrule{2-9}
		\multirow{6}{*}{\rotatebox{90}{\hspace*{-6pt}Kernel}}
		& \textsc{Graphlet}                  & 72.9           & 59.4                & 40.8                  & 58.3                & 72.1              & 87.7               & 54.7                \\
		& \textsc{Shortest-path}             & \textbf{76.4}  & 59.2                & 40.5                  & 62.1                & 74.5              & 81.7               & 58.9                \\
		& \textsc{$1$-WL}                    & 73.8           & 72.5                & \textbf{51.5}         & 62.9                & 83.1              & 78.3               & 61.3                \\
		& \textsc{$2$-WL}                    & 75.2           & 72.6                & 50.6                  & \textbf{64.7}       & 77.0              & 77.0               & 61.9                \\
		& \textsc{$3$-WL}                    & 74.7           & 73.5                & 49.7                  & 61.5                & 83.1              & 83.2               & 62.5                \\
		& \textsc{WL-OA}                     & 75.3           & 73.1                & 50.4                  & 62.7                & \textbf{86.1}     & 84.5               & \textbf{63.6}       
		\\
		\cmidrule{2-9}
		\multirow{7}{*}{\rotatebox{90}{GNN}}
		& \textsc{DCNN}                      & 61.3           & 49.1                & 33.5                  & ---                 & 62.6              & 67.0               & 56.6                \\
		& \textsc{PatchySan}                 & 75.9           & 71.0                & 45.2                  & ---                 & 78.6              & \textbf{92.6}      & 60.0                \\
		& \textsc{DGCNN}                     & 75.5           & 70.0                & 47.8                  & ---                 & 74.4              & 85.8               & 58.6                \\     
		\cmidrule{2-9}   
		
		& \textsc{$1$-Gnn No Tuning}         & 70.7           & 69.4                & 47.3                  & 59.0                & 58.6              & 82.7               & 51.2                \\
		& \textsc{$1$-Gnn}                   & 72.2           & 71.2                & 47.7                  & 59.3                & 74.3              & 82.2               & 59.0                \\
		& \textsc{$1$-$2$-$3$-Gnn No Tuning} & 75.9           & 70.3                & 48.8                  & 60.0                & 67.4              & 84.4               & 59.3                \\
		& \textsc{$1$-$2$-$3$-Gnn}           & 75.5           & \textbf{74.2}       & 49.5                  & 62.8                & 76.2              & 86.1               & 60.9                \\
		\bottomrule
	\end{tabular}
\end{table*}

\begin{table*}[t]\
	\caption{%
		Mean absolute errors on the \textsc{Qm9} dataset. The far-right column shows the improvement of the best $k$-GNN model in comparison to the $1$-GNN baseline.
	}%
	\label{fig:qm9_results}
	\renewcommand{\arraystretch}{1.0}
	\centering
	\begin{tabular}{lccccc}
		\toprule
		\multirow{3}{*}{\vspace*{8pt}\textbf{Target}} & \multicolumn{5}{c}{\textbf{Method}} \\
		\cmidrule{2-6}
		&   \textsc{$1$-Gnn}   & \!\textsc{$1$-$2$-Gnn} & \textsc{$1$-$3$-Gnn} & \textsc{$1$-$2$-$3$-Gnn}\! & Gain   \\
		\midrule
		$\mu$                                      & 0.493              & 0.493                  & $\mathbf{0.473}$     & 0.476                      & 4.0\%  \\
		$\alpha$                                                   & 0.78               & $\mathbf{0.27}$        & 0.46                 & $\mathbf{0.27}$            & 65.3\% \\
		$\varepsilon_{\text{HOMO}}$            & $\mathbf{0.00321}$ & 0.00331                & 0.00328              & 0.00337                    & --     \\
		$\varepsilon_{\text{LUMO}}$                  & 0.00355            & $\mathbf{0.00350}$     & 0.00354              & 0.00351                    & 1.4\%  \\
		$\Delta\varepsilon$                           & 0.0049             & 0.0047                 & $\mathbf{0.0046}$    & 0.0048                     & 6.1\%  \\
		$\langle R^2 \rangle$                       & 34.1               & 21.5                   & 25.8                 & 22.9                       & 37.0\% \\
		\textsc{ZPVE}                      & 0.00124            & $\mathbf{0.00018}$     & 0.00064              & 0.00019                    & 85.5\% \\
		$U_0$                                         & 2.32               & $\mathbf{0.0357}$      & 0.6855               & 0.0427                     & 98.5\% \\
		$U$                                                 & 2.08               & $\mathbf{0.107}$       & 0.686                & 0.111                      & 94.9\% \\
		$H$                                                      & 2.23               & 0.070                  & 0.794                & $\mathbf{0.0419}$          & 98.1\% \\
		$G$                                              & 1.94               & 0.140                  & 0.587                & $\mathbf{0.0469}$          & 97.6\% \\
		$C_{\text{v}}$                                & 0.27               & 0.0989                 & 0.158                & $\mathbf{0.0944}$          & 65.0\% \\
		\bottomrule
	\end{tabular}
\end{table*}

\section{Experimental Study}

In the following, we want to investigate potential benefits of GNNs over graph kernels as well as the benefits of our proposed $k$-GNN architectures over $1$-GNN architectures. More precisely, we address the following questions:
\begin{description}
	\item[Q1] How do the (hierarchical) $k$-GNNs perform in comparison to state-of-the-art graph kernels? 
	\item[Q2] How do the (hierarchical)  $k$-GNNs perform in comparison to the $1$-GNN in graph classification and regression tasks?
	\item[Q3]
	How much (if any) improvement is provided by optimizing the parameters of the GNN aggregation function, compared to just using random GNN parameters while optimizing the parameters of the downstream classification/regression algorithm?
\end{description}

\subsection{Datasets }
To compare our $k$-GNN architectures to kernel approaches we use well-established benchmark datasets from the graph kernel literature~\cite{KKMMN2016}. The nodes of each graph in these dataset is annotated with (discrete) labels or no labels. 

To demonstrate that our architectures scale to larger datasets and offer benefits on real-world applications, we conduct experiments on the \textsc{Qm9} dataset~\cite{Ram+2014,Rud+2012,Wu+2018}, which consists of 133\,385 small molecules. The aim here is to perform regression on twelve targets representing energetic, electronic, geometric, and thermodynamic properties, which were computed using density functional theory.

\subsection{Baselines}

We use the following kernel and GNN methods as baselines for our experiments.

\xhdr{Kernel Baselines} We use the Graphlet kernel~\cite{She+2009}, the shortest-path kernel~\cite{Borgwardt2005}, the Weisfeiler-Lehman subtree kernel (\textsc{WL})~\cite{She+2011}, the Weisfeiler-Lehman Optimal Assignment kernel (\textsc{WL-OA})~\cite{Kri+2016}, and the global-local $k$-WL~\cite{Mor+2017} with $k$ in $\{2,3\}$ as kernel baselines. 
For each kernel, we computed the normalized Gram matrix. 
We used the $C$-SVM implementation of LIBSVM~\cite{Cha+2011} to compute the classification accuracies using 10-fold cross validation. 
The parameter $C$ was selected from $\{10^{-3}, 10^{-2}, \dotsc, 10^{2},$ $10^{3}\}$ by 10-fold cross validation on the training folds. 

\xhdr{Neural Baselines} To compare GNNs to kernels we used the basic $1$-GNN layer of~\cref{eq:basicgnn}, DCNN~\cite{Wang2018}, PatchySan~\cite{Nie+2016}, DGCNN~\cite{Zha+2018}. For the \textsc{Qm9} dataset we used a $1$-GNN layer similar to~\cite{Gil+2017}, where we replaced the inner sum of~\cref{eq:basicgnn} with a 2-layer MLP in order incorporate edge features (bond type and distance information).

\subsection{Model Configuration}

We always used three layers for $1$-GNN, and two layers for (local) $2$-GNN and $3$-GNN, all with a hidden-dimension size of $64$. 
For the hierarchical variant we used architectures that use features computed by $1$-GNN as initial features for the $2$-GNN ($1$-$2$-GNN)  and $3$-GNN ($1$-$3$-GNN), respectively. 
Moreover, using the combination of the former we componentwise concatenated the computed features of the $1$-$2$-GNN and the $1$-$3$-GNN ($1$-$2$-$3$-GNN).
For the final classification and regression steps, we used a three layer MLP, with binary cross entropy and mean squared error for the optimization, respectively.
For classification we used a dropout layer with $p=0.5$ after the first layer of the MLP. 
We applied global average pooling to generate a vector representation of the graph from the computed node features for each $k$. 
The resulting vectors are concatenated column-wise before feeding them into the MLP. 
Moreover, we used the Adam optimizer with an initial learning rate of $10^{-2}$ and applied an adaptive learning rate decay based on validation results to a minimum of $10^{-5}$.  We trained the classification networks for $100$ epochs and the regression networks for $200$ epochs.

\subsection{Experimental Protocol} 

For the smaller datasets, which we use for comparison against the kernel methods, we performed a 10-fold cross validation where we randomly sampled 10\% of each training fold to act as a validation set.
For the \textsc{Qm9} dataset, we follow the dataset splits described in~\cite{Wu+2018}.
We randomly sampled 10\% of the examples for validation, another 10\% for testing, and used the remaining for training. We used the same initial node features as described in~\cite{Gil+2017}. Moreover, in order to illustrate the benefits of our hierarchical $k$-GNN architecture, we did not use a complete graph, where edges are annotated with pairwise distances, as input.
Instead, we only used pairwise Euclidean distances for connected nodes, computed from the provided node coordinates. The code was built upon the work of~\cite{Fey+2018} and is provided at~\url{https://github.com/chrsmrrs/k-gnn}.

\subsection{Results and Discussion}

In the following we answer questions \textbf {Q1} to \textbf{Q3}. \cref{fig:classification_results} shows the results for comparison with the kernel methods on the graph classification benchmark datasets. Here, the hierarchical $k$-GNN is on par with the kernels despite the small dataset sizes (answering question \textbf{Q1}).
We also find that the 1-2-3-GNN significantly outperforms the 1-GNN on all seven datasets (answering \textbf{Q2}), with the 1-GNN being the overall weakest method across all tasks.\footnote{Note that in very recent work, GNNs have shown superior results over kernels when using advanced pooling techniques~\cite{Yin+2018}. Note that our layers can be combined with these pooling layers. However, we opted to use standard global pooling in order to compare a typical GNN implementation with standard off-the-shelf kernels.}
We can further see that optimizing the parameters of the aggregation function only leads to slight performance gains on two out of three datasets, and that no optimization even achieves better results on the \textsc{Proteins} benchmark dataset (answering \textbf{Q3}).
We contribute this effect to the one-hot encoded node labels, which allow the GNN to gather enough information out of the neighborhood of a node, even when this aggregation is not learned.

\cref{fig:qm9_results} shows the results for the \textsc{Qm9} dataset. On eleven out of twelve targets all of our hierarchical variants beat the $1$-GNN baseline, providing further evidence for \textbf{Q2}. 
However, the additional structural information extracted by the $k$-GNN layers does not serve all tasks equally, leading to huge differences in gains across the targets.

It should be noted that our $k$-GNN models have more parameters than the $1$-GNN model, since we stack two additional GNN layers for each $k$. However, extending the $1$-GNN model by additional layers to match the number of parameters of the $k$-GNN did not lead to better results in any experiment.

\section{Conclusion}

We presented a theoretical investigation of GNNs, showing that a wide class of GNN architectures cannot be stronger than the $1$-WL. On the positive side, we showed that, in principle, GNNs possess the same power in terms of distinguishing between non-isomorphic (sub-)graphs, while having the added benefit of adapting to the given data distribution. Based on this insight, we proposed $k$-GNNs which are a generalization of GNNs based on the $k$-WL. This new model is strictly stronger then GNNs in terms of distinguishing non-isomorphic (sub-)graphs and is capable of distinguishing more graph properties. Moreover, we devised a hierarchical variant of $k$-GNNs, which can exploit the hierarchical organization of most real-world graphs. Our experimental study shows that $k$-GNNs consistently outperform $1$-GNNs. Future work includes designing task-specific $k$-GNNs, e.g., devising $k$-GNNs layers that exploit expert-knowledge in bio- and chemoinformatic settings. 

\section*{Acknowledgments}
This work is supported by the German research council (DFG) within the Research Training Group 2236 \emph{UnRAVeL} and the Collaborative Research Center
SFB 876, \emph{Providing Information by Resource-Constrained
	Analysis}, projects A6 and B2.

\fontsize{9.5pt}{10.5pt} \selectfont
\bibliographystyle{aaai}

\onecolumn

\fontsize{11.0pt}{14.0pt} \selectfont

\section{\LARGE Appendix}

In the following we provide proofs for Theorem 1, Theorem 2, Proposition 3, and Proposition 4.

\section{Proof of Theorem 1}

\begin{theorem}[Theorem 1 in the main paper]\label{thm:refine:restated}
	Let $(G, l)$ be a labeled graph. Then for all $t\ge 0$ and for all choices of initial colorings $f^{(0)}$ consistent with $l$, and weights $\mathbf{W}^{(t)}$,
	\begin{equation*}\label{refine}
		c^{(t)}_l \sqsubseteq f^{(t)}\,.
	\end{equation*}
\end{theorem}
For the theorem we consider a single iteration of the $1$-WL algorithm and the GNN on a single graph.
\begin{proof}[Proof of Theorem 1]
	We show for an arbitrary iteration $t$ and nodes $u,v\in V(G)$, that $c_l^{(t+1)}(u) = c_l^{(t+1)}(v)$ implies $f^{(t+1)}(u)=f^{(t+1)}(v)$. In iteration $0$ we have $c_l^{(0)}(u)=c_l^{(0)}(v) \Longleftrightarrow f^{(0)}(u)=f^{(0)}(v)$ as the initial node coloring $f^{(0)}$ is chosen consistent with $l$.
				
	Let $u,v\in V(G)$ and $t\in \NN$ such that $c_l^{(t+1)}(u) = c_l^{(t+1)}(v)$.
	Assume for the induction that $c_l^{(t)}(u)=c_l^{(t)}(v) \Longrightarrow f^{(t)}(u)=f^{(t)}(v)$ holds.
	As $c_l^{(t+1)}(u)= c_l^{(t+1)}(v)$ we know from the refinement step of the $1$-WL that the old colors $c_l^{(t)}(u) = c_l^{(t)}(v)$ of $u$ and $v$ as well as the multisets $\oms c_l^{(t)}(w) \mid w \in N(u) \cms$ and $\oms c_l^{(t)}(w) \mid w \in N(v) \cms$ of colors of the neighbors of $u$ and $v$ are identical.
				
	Let $M_u = \oms f^{(t)}(w) \mid w \in N(u) \cms$ and $M_v = \oms f^{(t)}(v) \mid w \in N(v) \cms$ be the multisets of feature vectors of the neighbors of $u$ and $v$ respectively. By the induction hypothesis, we know that $M_u = M_v$ and $f^{(t)}(u) = f^{(t)}(v)$ such that independent of the choice of $f_\text{merge}$ and $f_\text{aggr}$ we get $f^{(t+1)}(u) = f^{(t+1)}(v)$. This holds as the input to both functions $f_\text{merge}$ and $f_\text{aggr}$ is identical. This proves $c_l^{(t+1)}(u) = c_l^{(t+1)}(v) \Longrightarrow f^{(t+1)}(u)=f^{(t+1)}(v)$ and thereby the theorem.
\end{proof}

\section{Proof of Theorem 2}

\begin{theorem}[Theorem 2 in the main paper]\label{equal:restated}
	Let $(G, l)$ be a labeled graph. Then for all $t\geq 0$ there exists a sequence of weights $\mathbf{W}^{(t)},$ and a $1$-GNN architecture such that 
	\begin{equation*}
		c^{(t)}_l \equiv f^{(t)}\,.
	\end{equation*}
\end{theorem}
For the proof we start by giving the proof for graphs where all nodes have the same initial color and then extend it to colored graphs. 
In order to do that we use a slightly adapted but equivalent version of the $1$-WL.
Note that the extension to colored graphs is mostly technical, while the important idea is already contained in the first case.

\subsection{Uncolored Graphs}
Let $\Gamma_G$ be the refinement operator for the $1$-WL, mapping the old coloring $c_{l}^{(t-1)}$ to the updated one $c_{l}^{(t)}$:
\begin{equation*}\label{eq:wlColoring}
	c_{l}^{(t)}(v) = \left(\Gamma_G\big(c_{l}^{(t-1)}\big)\right)(v)= \textsc{hash}\Big(\! \big(c_l^{(t-1)}(v),\oms c_l^{(t-1)}(u)\!\mid\!u \in N(v) \cms \big)\! \Big)\,.
\end{equation*}
We first show that for uncolored graphs this is equivalent to the update rule $\tilde\Gamma_G$:
\begin{equation*}\label{eq:wlColoring}
	\tilde c_{l}^{(t)}(v) = \left(\tilde\Gamma_G\big(\tilde c_{l}^{(t-1)}\big)\right)(v)= \textsc{hash}\Big(\! \big(\oms c_l^{(t-1)}(u)\!\mid\!u \in N(v) \cms \big)\! \Big)\,.
\end{equation*}
We denote $J$ as the all-$1$ matrix where the size will always be clear from the context.

\begin{lemma}\label{lem:1}
	Let $G$ be a graph, $v,w\in V(G)$, and $t\in\NN$ such that
	$\tilde c^{(t)}(u) \neq \tilde c^{(t)}(v)$.
	Then $\tilde c^{(t')}(u) \neq \tilde c^{(t')}(v)$ for all $t'\geq t$.
\end{lemma}

\begin{proof}
	Let $t\in\NN$ be minimal such that there are $v,w$ with
	\begin{align}
		\label{eq:2}                               
		\tilde c^{(t)}(u)\neq\tilde c^{(t)}(v),    
		\intertext{and}                            
		\label{eq:1}                               
		\tilde c^{(t+1)}(u)=\tilde c^{(t+1)}(v)\,. 
	\end{align}
	Then $t\ge 1$, because $\tilde c^{(0)}=J$ as there are no initial colors.
	Let $P_1,\ldots,P_p$ be the color classes of $\tilde c^{(t-1)}$. That
	is, for all $x,y\in V(G)$ we have
	$\tilde c^{(t-1)}(x)=\tilde c^{(t-1)}(y)$ if any only if there is an
	$i\in[1:p]$ such that $x,y\in P_i$. 
	Similarly, let $Q_1,\ldots,Q_q$ be the color classes of $\tilde c^{(t)}$. 
	Observe that the partition $\{Q_1,\ldots,Q_q\}$ of $V(G)$ refines the partition $\{P_1,\ldots,P_p\}$. 
	Indeed, if there were $i\neq i'\in [1\!\!:\!\!p]$, $k\in[1\!\!:\!\!q]$ such that $P_i\cap Q_k\neq\emptyset$ and $P_{i'}\cap Q_k\neq\emptyset$, then all $x\in P_i\cap Q_k$, $y\in P_{i'}\cap Q_k$ would satisfy $\tilde c^{(t-1)}(x)\neq\tilde c^{(t-1)}(y)$ and $\tilde c^{(t)}(x)=\tilde c^{(t)}(y)$, contradicting the minimality of $q$.
				
	Choose $v,w\in V(G)$ satisfying \eqref{eq:2} and \eqref{eq:1}. 
	By \eqref{eq:2}, there is an $i\in[1\!\!:\!\!p]$ such that $|N_G(v)\cap P_i|\neq |N_G(w)\cap P_i|$. 
	Let $j_1,\ldots,j_\ell\in[1\!\!:\!\!q]$ such that $P_i=Q_{j_1}\cup\ldots\cup Q_{j_\ell}$. 
	By \eqref{eq:1}, for all $k\in[1:\ell]$ we have $|N_G(v)\cap Q_{j_k}|= |N_G(w)\cap Q_{j_k}|$. 
	As the $Q_j$ are disjoint, this implies $|N_G(v)\cap P_i|=|N_G(w)\cap P_i|$ , which is a contradiction.
\end{proof}
Hence, the two update rules are equivalent.
\begin{corollary}
	For all $G$ and all $t\in\NN$ we have $\tilde c^{(t)}\equiv c^{(t)}$.
\end{corollary}
Thus we can use the update rule $\tilde\Gamma$ for the proof on unlabeled graphs. For the proof, it will be convenient to 
assume that $V(G) = [1\!\!:\!\!n]$ (although we still work with the notation $V(G)$). It follows that $n = |V|$.  A node coloring $f^{(t)}$ defines a matrix $F^{(t)}\in \RR^{n\times d}$ where the $i^\text{th}$ row of $F^{(t)}$ is defined by $f^{(t)}(i)\in \RR^{1 \times d}$.
Here we interpret $i$ as a node from $V(G)$. As colorings and matrices can be interpreted as one another, given a matrix $F\in \RR^{V(G)\times d}$ we write $\Gamma_G(F)$ (or $\tilde\Gamma_G(F)$) for a Weisfeiler-Leman iteration on the coloring $c_F$ induced by the matrix $F$. For the GNN computation we provide a matrix based notation.
Using the adjacency matrix $A\in \RR^{n \times n}$ of $G$ and a coloring $F\in \RR^{n \times d}$, we can write the update rule of the GNN layer as
\begin{equation*}
	F^{t+1} = \Lambda_{A,W,b}(F^{(t)}) = \sigma( AF^{(t)}W^{(t)} + bJ),
\end{equation*}
where $\Lambda_{A,W,b}$ is the refinement operator of GNNs corresponding to a single iteration of the $1$-WL. For simplicity of the proof, we choose
\begin{equation*}
	\sigma=\sgn = 
	\begin{cases}
		1  & \text{if }x>0,    \\
		-1 & \text{otherwise}, 
	\end{cases}
\end{equation*}
and the bias as 
\begin{equation*}
	b =-1\,. 
\end{equation*}
Note that we later provide a way to simulate the sign-function using ReLu operations to indicate that choosing the sign function is not really a hard restriction.
\begin{lemma}
	\label{lem:dist2lu}
	Let $B\in\ZZ^{s\times t}$ be a matrix such that $0\le B_{ij}\le n-1$
	for all $i,j$ and the rows of $B$ are pairwise distinct. 
	Then there is a matrix $X\in\RR^{t\times s}$ such that the matrix
	$\sgn\big(BX-J)\in\{-1,1\}^{s\times s}$ is non-singular.
\end{lemma}

\begin{proof}
	Let $z=(1,n,n^2,\ldots,n^{t-1})^T\in \RR^t$ where $n$ is the upper bound on the matrix entries of $B$  and $b=B z\in\RR^s$. 
	Then the entries of $b$ are nonnegative and pairwise distinct.
	Without loss of generality, we assume that $\vec b=(b_1,\ldots,b_s)^T$ such that $b_1>b_2>\cdots >b_s\ge 0$. Now we choose numbers $x_{1},\ldots,x_s\in\RR$ such that
	\begin{equation}
		\label{eq:25}
		\begin{cases}
			b_{i}\cdot x_j<1 & \text{if }i\ge j, \\
			b_i\cdot x_j>1   & \text{if }i<j     
		\end{cases}
	\end{equation}
	for all $i,j\in[s]$ as the $b_i$ are ordered. 
	Let $\vec x=(x_1,\ldots,x_s)\in\RR^{1\times s}$ and $C=  b\cdot x\in\RR^{s\times s}$ and $\hat C =\sgn(C-J)$. 
	Then $C$ has entries $C_{ij}=b_i\cdot x_j$, and thus by \eqref{eq:25}, 
	\begin{equation}
		\label{eq:29}
		\hat C=
		\begin{pmatrix}
			-1     & 1  & 1 & 1      & \cdots & 1  & 1      \\
			-1     & -1 & 1 &        & \cdots &    & 1      \\
			\\
			\vdots &    &   & \ddots & \ddots &    & \vdots \\
			\\
			-1     &    &   & \cdots &        & -1 & 1      \\
			-1     &    &   & \cdots &        & -1 & -1     
		\end{pmatrix}
	\end{equation}
	Thus $\hat C$ is non-singular. Now we simply let $X= z\cdot x$. Then $BX=C$.
\end{proof}
Let us call a matrix \emph{row-independent modulo equality} if the set
of all rows appearing in the matrix is linearly independent.

\begin{example}
	The matrix
	\[
		\begin{pmatrix}
			1 & 1 & 0 \\
			1 & 1 & 0 \\
			0 & 1 & 0 \\
			0 & 1 & 0 \\
			0 & 1 & 0 
		\end{pmatrix}
	\]
	is row-independent modulo equality.
\end{example}
Note that the all-$1$ matrix $J$ is row-independent modulo equality in all dimensions.
\begin{lemma}\label{lem:equivalencePerStep}
	Let $d\in\NN$, and let  $F\in\RR^{n\times d}$ be row independent modulo equality. 
	Then there is a $W\in\RR^{d\times n}$ such that the matrix $\Lambda_{A,W,-1}(F)$ is row independent modulo equality and
	\begin{equation*}
		\Lambda_{A,W,-1}(F)\equiv \tilde\Gamma_G(F)\,.
	\end{equation*}
\end{lemma}
\begin{proof}
	Let $Q_1,\ldots,Q_r$ be the color
	classes of $F$ (that is, for all $v,v'\in V(G)$ it holds that
	$F_{v}=F_{v'}\iff\exists j\in[r]:\; v,v'\in Q_j$). Let
	$\tilde F\in\RR^{r\times d}$ be the matrix with rows $\tilde
	F_{j}=F_{v}$ for all $j\in[r], v\in Q_j$. Then the rows of $\tilde
	F$ are linearly independent, and thus there is a matrix 
	$M\in\RR^{d\times r}$ such that $\tilde FM$ is the $(r\times r)$ identity
	matrix. 
	It follows that $FM\in\RR^{n\times r}$ is the matrix with entries
		
	\begin{equation}
		\label{eq:12}
		(FM)_{vj}=
		\begin{cases}
			1 & \text{if }v\in Q_j, \\
			0 & \text{otherwise}.   
		\end{cases}
	\end{equation}
	Let $D\in\ZZ^{n\times r}$ be
	the matrix with entries $D_{vj}:=|N_G(v)\cap Q_j|$. Note that
	\begin{equation}
		\label{eq:13}
		AFM=D,
	\end{equation}
	because for all $v\in V$ and $j\in[t]$ we have
	\[
		(AFM)_{vj}=\sum_{v'\in V(G)}A_{vv'}(FM)_{v'j}
		= \sum_{v'\in Q_j}A_{vv'}=D_{vj},
	\]
	where the second equality follows from Equation \eqref{eq:12}. By the definition of $\Gamma_G$ as the $1$-WL operator on uncolored graphs, we have
	\begin{equation}
		\label{eq:10}
		\Gamma_G(F)\equiv D
	\end{equation}
	if we view $D$ as a coloring of $V$.
				
	Let $P_1,\ldots,P_s$ be the color classes of $D$, and let $\tilde D\in\ZZ^{s\times r}$ be the matrix
	with rows $\tilde D_{i}=D_{v}$ for all $i\in[s]$ and $v\in
	P_i$. Then $0\le\tilde D_{ij}\le n-1$ for all $i,j$, and the rows of
	$\tilde D$ are pairwise distinct. By Lemma~\ref{lem:dist2lu}, there is a
	matrix $X\in\RR^{r\times s}$ such that the matrix
	$\sgn(\tilde DX-J)\in\RR^{s\times s}$ is non singular. This implies
	that the matrix $\sgn(AFMX-J)=\sgn(DX-J)$ is row-independent modulo
	equality. Moreover, $\sgn(AFMX-J)\equiv D\equiv \Gamma_G(F)$ by \eqref{eq:10}. We let $W\in\RR^{p\times n}$ be the matrix of obtained from
	$MX\in\RR^{p\times s}$ by adding $n-s$ all-0 columns. Then 
	\[
		\Lambda_{A,W,-1}(F)=\sgn(AFW-J)
	\]
	is row-independent modulo equality and
	$\Lambda_{A,W,-1}(F)\equiv\sgn(AFMX-J)\equiv \tilde\Gamma_G(F)$. 
\end{proof}
\begin{corollary}
	There is a sequence $\vec W=(W^{(t)})_{t\in\NN}$ with
	$W^{(t)}\in \RR^{n\times n}$ such that for all $r\in\NN$,
	\[
		\tilde c^{(t)}\equiv \Lambda_{A,\vec W,-1}^{(t)}\,.
	\]
	where $\tilde c^{(t)}$ is given by the $t$-fold application of $\tilde\Gamma_G$ on the initial uniform coloring $J$.
\end{corollary}
\begin{remark}\label{rem:fixedSizeOutput}
	The construction in \cref{lem:equivalencePerStep} always outputs a matrix with as many columns as there are color classes in the resulting coloring.
	Thus we can choose $d$ to be $n$ and pad the matrix using additional $0$-columns.
\end{remark}

\subsection{Colored Graphs}
We now extend the computation to colored graphs.
In order to do that, we again use an equivalent but slightly different variant of the Weisfeiler-Leman update rule leading to colorings $c_{l,0}^{(t)}$ instead of the usual $c_{l}^{(t)}$. 
We then start by showing that both update rules are equivalent.

We define $\Gamma_{G,l}$ to be the refinement operator for the $1$-WL, mapping a coloring $c_{l,0}^{(t-1)}$ to the updated one $c_{l,0}^{(t)}$ as follows:
\begin{equation*}\label{eq:4}
	c_{l,0}^{(t)}(v) = \left(\Gamma_{G,l}\big(c_{l,0}^{(t-1)}\big)\right)(v)= \textsc{hash}\Big(\! \big(c_{l,0}^{(0)}(v),\oms c_{l,0}^{(t-1)}(u)\!\mid\!u \in N(v) \cms \big)\! \Big)\,.
\end{equation*}
Note that for $\Gamma_{G,l}$ we use the initial color $c_{l,0}^{(0)}(u)$ of a node $u$ whereas $\Gamma_G$ used the color $c_l^{(t-1)}(u)$ from the previous round.
The idea of using those old colors is to make sure that any two nodes which got a different color in iteration $t$, get different colors in iteration $t'>t$.
This is formalized by the following lemma.

\begin{lemma}\label{lem:2}
	Let $(G,l)$ be a colored graph, $v,w\in V(G)$, and $q\in\NN$ such that
	$c_{l,0}^{(t)}(v)\neq c^{(t)}_{l,0}(w)$. 
	Then $c^{(t')}_{l,0}(v)\neq c^{(t')}_{l,0}(w)$ for all $t'\geq t$.
\end{lemma}

\begin{proof}
	Let $t\in\NN$ be minimal such that there are $v,w$ with
	\begin{align}
		\label{eq:5}
		c^{(t)}_{l,0}(v)   & \neq c^{(t)}_{l,0}(w) 
		\intertext{and}
		\label{eq:6}
		c^{(t+1)}_{l,0}(v) & =c^{(t+1)}_{l,0}(w).  
	\end{align}
	Then $t\ge 1$, because by \eqref{eq:4},
	$c^{(1)}_{l,0}(v)=c^{(1)}_{l,0}(w)$ implies $c^{(0)}_{l,0}(v)=c^{(0)}_{l,0}(w)$. Let $P_1,\ldots,P_p$ be the color classes of $c_{l,0}^{(t-1)}$, and let $Q_1,\ldots,Q_q$ be the color classes of $c_{l,0}^{(t)}$. 
	Observe that the partition $\{Q_1,\ldots,Q_q\}$ of $V(G)$ refines the partition
	$\{P_1,\ldots,P_p\}$. The argument is the same as in the proof of
	Lemma~\ref{lem:1}.
		
	Choose $v,w\in V(G)$ satisfying \eqref{eq:5} and \eqref{eq:6}. 
	By \eqref{eq:5}, either $c_{l,0}(v)\neq c_{l,0}(w)$ or there is an $i\in[1:p]$ such that $|N_G(v)\cap P_i|\neq |N_G(w)\cap P_i|$. 
	By \eqref{eq:4}, $c_{l,0}(v)\neq c_{l,0}(w)$ contradicts \eqref{eq:6}. 
	Thus $|N_G(v)\cap P_i|\neq |N_G(w)\cap P_i|$ for some $i\in[1\!\!:\!\!p]$. 
	Let $j_1,\ldots,j_\ell\in[1\!\!:\!\!q]$ such that $P_i=Q_{j_1}\cup\ldots\cup Q_{j_\ell}$. By \eqref{eq:5}, for all $k\in[1\!\!:\!\!\ell]$ we have $|N_G(v)\cap Q_{j_k}|= |N_G(w)\cap Q_{j_k}|$. 
	As the $Q_j$ are disjoint, this implies $|N_G(v)\cap P_i|=|N_G(w)\cap P_i|$, which is a contradiction.
\end{proof}
\begin{corollary}\label{cor:c-l-zero}
	For all graphs $G$ and initial vertex colorings $l$ of $G$ we have
	$c_{l,0}^{(t)}\equiv c_{l}^{(t)}$ for all $t\in\NN$.
\end{corollary}
We now consider the slightly modified update rule for $1$-GNNs which takes the initial colors $l$ into account.
Let the matrix $F^{(0)}_{l,0}\in \RR^{n\times d}$ be an encoding (such as the one-hot encoding) of $l$ such that $F^{(0)}_{l,0}$ is linearly independent modulo equality. Then
\begin{equation*}
	F^{(t+1)}_{l,0} = \Lambda_{A,0,\vec W,\vec b}(F^{(t)}_{l,0})=\Big(F^{(0)}_{l,0}, \;\;\Lambda_{A,W,b}\big(\Lambda_{A,0,\vec W,b} (F^{(t)}_{l,0})\big)\Big)\,.
\end{equation*}

We now show that this version of GNNs, which can be implemented by the simple variant (1) of GNNs provided in the main paper, is equivalent to $1$-WL on colored graphs, proving the main theorem.

\begin{proof}[Proof of Theorem 2 (sketch)]
	We prove the theorem by induction over the iteration $t$.
	The initial colorings are chosen consistent with $l$ providing  $c_{l,0}^{(0)} \equiv F^{(0)}_{l,0}$.
	For the induction step we assume $c_{l,0}^{(t)} \equiv F^{(t)}_{l,0}$ for iteration $t$.
	We know by \cref{lem:equivalencePerStep} that the inner part $\Lambda_{A,W,b}\big(\Lambda_{A,0,\vec W,b} (F^{(t)}_{l,0})\big)$ of the update rule results in color classes which are identical to the ones that $\tilde\Gamma_G$ would produce.
	This implies that restricted to each color class $Q$ from $Q_1^{(t)},\dots,Q_q^{(t)}$ of iteration $t$, the new color classes $Q_1^{(t+1)}\cap Q,\dots,Q_{q'}^{(t+1)}\cap Q$ restricted to $Q$ match the coloring $c_{l,0}^{(t+1)}|_Q$ that is $c_{l,0}^{(t+1)}$ restricted to $Q$.
	This holds as within one color class, the common color of the nodes contains no further information as shown in \cref{lem:1}.
	Observe that colors are represented by linearly independent row vectors.
	This especially holds for $F^{(t+1)} = \Lambda_{A,W,b}\big(\Lambda_{A,0,\vec W,b} (F^{(t)}_{l,0})\big)$.
	In order to show that $F^{(t+1)}_{l,0}$ represents the coloring $c^{(t+1)}_{l,0}$ we have to prove two properties.
	\begin{enumerate}
		\item Given $Q_i^{(0)} \neq Q_j^{(0)}$ and $u\in Q_i^{(0)}, v\in Q_j^{(0)}$, we have $F^{(t+1)}_{l,0}(u)\neq F^{(t+1)}_{l,0}(v)$.
		\item $F^{(t+1)}_{l,0}$ is linearly independent modulo equality.
	\end{enumerate}
	In the first item we interpret $F^{(t+1)}_{l,0}$ as a coloring function. Both properties follow directly from the definition of $F^{(t+1)}_{l,0}$ as the concatenation of the matrices $F^{(0)}_{l,0}$ and $F^{(t+1)}$.
	The first property holds as the row vectors of $F^{(0)}_{l,0}$ are clearly different, as otherwise $Q_i=Q_j$.
	The second property follows from the fact that linear independence cannot be lost by extending a matrix.
	Thus within each old color class $Q_i^{(t)}$, all vectors are linearly independent modulo equality as $F^{(t+1)}$ (without subscript) is linearly independent modulo equality.
	This then extends to all combinations of old color classes and colors from $F^{(t+1)}$ as $F^{(0)}_{l,0}$ is also linearly independent modulo equality.
	By induction this chows that $c_{l,0}^{(t)} \equiv F^{(t)}_{l,0}$ for all $t$.
				
	Note that the width of the matrices $F^{(t+1)}_{l,0}$ can be bounded by $2n$.
	The width of $F^{(0)}_{l,0}$ can trivially be bounded by $n$ (in the worst case every node has a different initial color).
	The width of $F^{(t+1)}$ can also be bounded by $n$ using \cref{rem:fixedSizeOutput} as it is the output of a computation using \cref{lem:equivalencePerStep}. Using \cref{cor:c-l-zero} for the step from $c^{(t+1)}_{l,0}$ to $c^{(t+1)}_{l}$ finishes the proof of 2.
\end{proof}

\subsection{Adaptation to the ReLu activation function.}
The above framework can be easily adapted to 
use the ReLu function as the activation function 
of our choice, as follows. Recall that the $\sgn$ activation function 
was applied to the matrix $C-J$,
where the matrix $C$ is defined via Equation \ref{eq:25}.  
We claim that a two-fold application of ReLu, interspersed 
by some elementary matrix operations, also yields the matrix $\sgn(C-J)$.
This ensures that we can use our GNN architecture to work with the ReLu activation function. 

The proof is as follows. 
The first application of ReLu to the matrix $-(C-J)$ yields the matrix $C^{(1)}:= \sigma_\text{ReLu}(J-C)$,
which is a lower triangular matrix satisfying
\begin{equation*}
	\begin{cases}
		C^{(1)}_{ij} = 0 & \text{for } i <j \text{ and} \\
		C^{(1)}_{ij} > 0 & \text{for }i \geq j          
	\end{cases}
\end{equation*}
Let $\delta$ be the smallest positive value occurring in $C^{(1)}$.
Note that $\delta$ is well defined as there are only a finite number of possible binary input matrices (with values $\{-1,1\}$) for which $C^{(1)}$ can deterministically be computed.
The matrix $C^{(2)} := -\frac{2}{\delta}C^{(1)} +2J$ then satisfies 
\begin{equation*}
	\begin{cases}
		C^{(2)}_{ij} = 2 & \text{for } i<j \text{ and} \\
		C^{(2)}_{ij} < 0 & \text{for } i \geq j.       
	\end{cases}
\end{equation*}
Another application of ReLu yields an upper triangular matrix $C^{(3)} := \sigma_\text{ReLu}(C^{(2)})$,
where every non-zero entry is equal to $2$. Subtracting $J$ from $C^{(3)}$ then yields $\sgn{}(C)$. We now show that those operations can be represented by the GNN architecture.
Recall that our GNN architecture was given by stacking a number of iterations of the form, in matrix notation
\begin{equation}\label{eq:matrixGNN}
	F^{(t)} = \sigma \Big( F^{(t-1)}W^{(t)}_1 + AF^{(t-1)}_l W^{(t)}_2 + bW^{(t)}_3 \Big)
\end{equation}
where the $W^{(t-1)}_3$ term is the bias that is added in each iteration and $\sigma$ is a non-linear function such as $\operatorname{tanh}$ or ReLu written as $\sigma_{\operatorname{ReLu}}$.

To compute $\sigma_{\operatorname{ReLu}}(J-C)$ we can choose $W^{(1)}_1$ and $W^{(1)}_2$ appropriate to reach $-C$, $W^{(1)}_3 = J$ as the all-$1$ matrix and $b=1$.
Correspondingly, for $\sigma_{\operatorname{ReLu}}(-\frac{2}{\delta}C^{(1)} +2J)$ we can choose $W^{(2)}_1 = -\frac{2}{\delta}I$, $b=2$, $W^{(2)}_3 = J$ and $W^{(2)}_2 = \vec 0$.
In the next round $W^{(3)}_3$ can be chosen according to the regular choices of $W^{(3)}_1$ and $W^{(3)}_2$ to simulate a computation on $\sgn{}(C)$ instead of $\sgn{}(C+J)$ which works by the linearity of the inner computation. Every layer described by Equation \eqref{eq:matrixGNN} is followed by an adjustment layer plus a few modifications in both the previous and upcoming step.
As a result, we have a constant (two-fold) increase
in the number of layers in our GNN set-up. The following corollary 
summarizes the above discussion.

\begin{corollary}[ReLu activation]
	Let $(G, l)$ be a labeled graph.
	Then there exists a sequence $(\mathbf{W}^t)_{ t > 0}$ and a $1$-GNN architecture based on the simple architecture described in (5) in the main paper using ReLu as activation function, such that $c^t_l=f^{(2t)}$ for all $t\geq 0$.
\end{corollary}

\section{Proofs of Proposition 3 and Proposition 4}

\begin{proposition}[Proposition 3 in the main paper]
	Let $(G, l)$ be a labeled graph and let $k\geq 2$. Then for all $t \ge 0$, for all choices of initial colorings $f_k^{(0)}$ consistent with $l$ and for all weights $\mathbf{W}^{(t)}$,
	\begin{equation*}\label{refine}
		c^{(t)}_{\text{s},k,l} \sqsubseteq f^{(t)}_{k}\,. 
	\end{equation*}
\end{proposition}

\begin{proof}
	The proof follows the arguments in the proof of Theorem 1.
	We therefore only provide a brief proof by induction on the iteration $t$. 
	For the base case, i.e. iteration $t=0$, the statement holds because 
	the initial coloring $f_k^{(0)}$ is chosen to be consistent with the 
	isomorphism types $d^0_{k,l}$.
				
	For the inductive step, assume that the statement holds until iteration $t-1$.
	Consider two tuples $u$ and $v$ 
	which (i) are not distinguished in the first $(t-1)$ iterations 
	and (ii) are not distinguished in the $t^\text{th}$ iteration. 
	Therefore, $u$ and $v$ must have equal number of neighbors from 
	every color class. This implies that the $k$-GNN update rule yields the same output for $u$ and $v$.
	Hence, if two such tuples are distinguished by the $k$-GNN in the $t^\text{th}$ iteration,
	they must be distinguished by the $k$-WL as well. This finishes the induction and proves the proposition. 
\end{proof}

\begin{proposition}[Proposition 4 in the main paper]
	Let $(G, l)$ be a labeled graph and let $k\geq 2$. Then for all $t \geq 0$ there exists a sequence of weights $\mathbf{W}^{(t)}$ and a $k$-GNN architecture such that   
	\begin{equation*}
		c^{(t)}_{\text{s},k,l} \equiv f^{(t)}_{k}\,.
	\end{equation*}
\end{proposition}

\begin{proof}
	Let us simulate the set-based $k$-WL on a $n$-vertex graph $G$ 
	via a $1$-WL on a graph $G^{\otimes k}$ on $O(n^k)$ vertices,
	defined as follows. The vertex set of $G^{\otimes k}$ is the 
	set $[V(G)]^k$ of all $k$-element subsets of $V(G)$. 
	The edge set of $G^{\otimes k}$ is defined as follows: two 
	sets $s$ and $t$ are connected by an edge in $G^{\otimes k}$ 
	if and only if $|s \cap t| = k-1$. Observe that the 
	neighborhood of a vertex $s$ in this graph
	is exactly the set $N(s)$ defined earlier. 
	The initial labeling of the vertices of the graph $G^{\otimes k}$ is determined as follows:
	For $s\in V(G^{\otimes k})$ the initial label of $s$ is its isomorphism type.
				
	For the above construction, it immediately follows that 
	performing the $1$-WL on the graph $G^{\otimes k}$ yields the same 
	coloring, as the one obtained by performing $k$-WL for the graph $G$. It remains to define the sequence $(\boldsymbol{W}^{t})_{t>0}$ such that $k$-GNN simulates the set based $k$-WL on $G$. Applying Theorem 2 to the graph $G^{\otimes k}$ results in a 
	sequence $\widetilde{\boldsymbol{W}}^{t}$ such that the $1$-GNN can simulate $1$-WL on $G^{\otimes k}$ using $\widetilde{\boldsymbol{W}}^{t}$. 
	Hence, this sequence can be directly used in the $k$-GNN to simulate $k$-WL on $G$. 
\end{proof}


\begin{table*}
	\caption{%
	Details of the  \textsc{Qm9} dataset. $\dagger$: Including atomic energy/enthalpy.
	}%
	\label{fig:qm9_details}
	\renewcommand{\arraystretch}{1.0}
	\centering
	\begin{tabular}{lcc}
		\toprule
		\textbf{Property} & \textbf{Unit} & \textbf{Description} \\
		\midrule
		$\mu$                       & Debye   &     Dipole moment  \\
		$\alpha$                    &  Bohr$^3$   &    Isotropic polarizability \\
		$\varepsilon_{\text{HOMO}}$ & Hartree &     Energy of highest occupied molecular
		orbital (HOMO)  \\
		$\varepsilon_{\text{LUMO}}$ &  Hartree  &    Energy of lowest occupied molecular
		orbital (LUMO)     \\
		$\Delta\varepsilon_{\text{HOMO}}$         &  Hartree  &    Gap, difference between LUMO and \\
		$\langle R^2 \rangle$       &Bohr$^2$  &    Electronic spatial extent                \\
		\textsc{ZPVE}               &Hartree  &    Zero point vibrational energy    \\
		$U_0$    $^\dagger$                  &Hartree &     Internal energy at 0 K  \\
		$U$      $^\dagger$                   &Hartree   &   Internal energy at 298.15 K  \\
		$H$      $^\dagger$                   &Hartree     & Enthalpy at 298.15 K  \\
		$G$      $^\dagger$                   &Hartree     & Free energy at 298.15 K
		 \\
		$C_{\text{v}}$              & cal/(mol K)&  Heat capacity at 298.15 K   \\
		\bottomrule
	\end{tabular}
\end{table*}

\end{document}